\documentclass{article}
\RequirePackage[OT1]{fontenc}
\RequirePackage{amsthm,amsmath,amssymb}
\RequirePackage[colorlinks,citecolor=blue,urlcolor=blue]{hyperref}

\usepackage{epsfig}
\usepackage{color}
\usepackage{graphicx}
\usepackage{enumerate}
\usepackage{epstopdf}
\usepackage{fullpage}

\numberwithin{equation}{section}
\theoremstyle{plain}
\newtheorem{thm}{Theorem}[section]

\newtheorem{lemma}[thm]{Lemma}
\newtheorem{cor}[thm]{Corollary}

\newtheorem{defn}[thm]{Definition}

\newcommand{\R}{{\mathbb{R}}}

\newcommand{\X}{{\mathcal{X}}}
\newcommand{\C}{{\mathbb{C}}}

\newcommand{\G}{\mathcal{G}}
\renewcommand{\P}{\mathbb{P}}
\newcommand{\E}{\mathbb{E}}

\newcommand{\nn}{{\rm NN}}
\newcommand{\braces}[1]{\left\{#1\right\}}
\newcommand{\norm}[1]{\left\|#1\right\|}
\newcommand{\abs}[1]{\left|#1\right|}
\newcommand{\paren}[1]{\left(#1\right)}

\newcommand{\Xcal}{\mathcal{X}}

\usepackage[round,comma]{natbib}
\begin{document}

\title{Consistent procedures for cluster tree estimation and pruning}

\author{
Kamalika Chaudhuri \\ \texttt{kamalika@cs.ucsd.edu} 
\and 
Sanjoy Dasgupta \\ \texttt{dasgupta@cs.ucsd.edu} 
\and
Samory Kpotufe \\ \texttt{samory@ttic.edu}
\and 
Ulrike von Luxburg \\ \texttt{luxburg@informatik.uni-hamburg.de}
}

\maketitle

\begin{abstract}
For a density $f$ on $\R^d$, a {\it high-density cluster} is any
connected component of $\{x: f(x) \geq \lambda\}$, for some $\lambda > 0$.
The set of all high-density clusters forms a hierarchy called the
{\it cluster tree} of $f$. We present two procedures for estimating
the cluster tree given samples from $f$. The first is a robust
variant of the single linkage algorithm for hierarchical clustering.
The second is based on the $k$-nearest neighbor graph of the samples.
We give finite-sample convergence rates for these algorithms which also
imply consistency, and we derive lower bounds on the sample complexity
of cluster tree estimation. Finally, we study a tree pruning procedure
that guarantees, under milder conditions than usual, to remove clusters
that are spurious while recovering those that are salient.
\end{abstract}

\section{Introduction}

We consider the problem of hierarchical clustering in a ``density-based'' setting,
where a cluster is formalized as a region of high density. Given data drawn i.i.d.
from some unknown distribution with density $f$ in $\R^d$, the goal is to estimate
the ``hierarchical cluster structure'' of the density, where a cluster is defined
as a connected subset of an $f$-level set $\{x \in \Xcal\; : \; f(x) \geq \lambda\}$.
These subsets form an infinite tree structure as $\lambda \geq 0$ varies, in the
sense that each cluster at some level $\lambda$ is contained in a cluster at a
lower level $\lambda' < \lambda$. This infinite tree is called the \emph{cluster tree}
of $f$ and is illustrated in Figure~\ref{fig-ulesclustertree}.

\begin{figure}
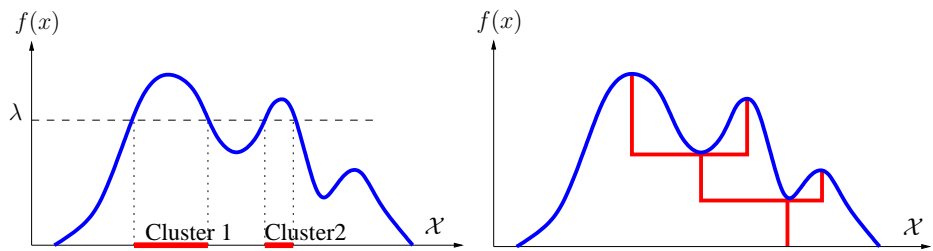

\centering
\resizebox{2.4in}{!}{\input{treeclusters.pstex_t}}
\resizebox{2.4in}{!}{\input{treeclusters2.pstex_t}}
\caption{Left: A probability density $f$ on $\R$, and two clusters at a fixed
level $\lambda$. Right: The same density, with the branching structure of the corresponding
cluster tree.}
\label{fig-ulesclustertree}
\end{figure}

Our formalism of the cluster tree (Section~\ref{ssec-clustertree}) and our notion
of consistency follow early work on clustering, in particular that of \citet{H81}.
Much subsequent work has been devoted to estimating the connected
components of a single level set; see, for example, \citet{P95b,T97}
and, more recently, \citet{MHL09, RV09}, \citet{RW10}, and \citet{SSN09}.
In contrast to these results,
the present work is concerned with the simultaneous estimation of all level sets
of an unknown density: recovering the cluster tree as a whole.

Are there hierarchical clustering algorithms which converge to the cluster
tree? Previous theory work \citep{H81, P95} has provided partial consistency
results for the well-known single-linkage clustering algorithm, while other
work~\citep{W69} has suggested ways to overcome the deficiencies of this
algorithm by making it more robust, but without proofs of convergence.
In this paper, we propose a novel way to make single-linkage more robust,
while retaining most of its elegance and simplicity (see Figure~\ref{fig:alg}).
We establish its finite-sample rate of convergence (Theorem~\ref{thm:upper-bound});
the centerpiece of our argument is a result on continuum percolation
(Theorem~\ref{thm:connectedness}). This also implies consistency in
the sense of Hartigan.

We then give an alternative procedure based on the $k$-nearest neighbor
graph of the sample (see Figure~\ref{fig:alg2}). Such graphs are widely
used in machine learning, and interestingly there is still much to understand
about their expressiveness. We show that by successively removing points from
this graph, we can create a hierarchical clustering that also converges to the
cluster tree, at roughly the same rate as the linkage-based scheme
(Theorem~\ref{thm:upper-bound2}).

Next, we use tools from information theory to give a lower bound on the problem
of cluster tree estimation (Theorem~\ref{thm:lower-bound}), which matches our
upper bounds in its dependence on most of the parameters of interest.

The convergence results for our two hierarchical clustering procedures nevertheless
leave open the possibility that the trees they produce contain spurious branching.
This is a well-studied problem in the cluster tree literature, and we address it
with a pruning method (Figure~\ref{fig:alg_pruning}) that preserves the consistency
properties of the tree estimators while providing finite-sample guarantees on the
removal of false clusters (Theorem~\ref{thm:pruning}). This procedure
is based on simple intuition that can carry over to other cluster tree estimators.

\section{Definitions and previous work}\label{sec-def}

Let $\X$ be a subset of $\R^d$. We exclusively consider Euclidean distance on
$\X$, denoted $\| \cdot \|$. Let $B(x,r)$ be the closed ball of radius $r$
around $x$.

\subsection{Clustering}
We start by considering the more general context of clustering. While clustering
procedures abound in statistics and machine learning, it remains largely unclear
whether clusters in finite data---for instance, the clusters returned by a
particular procedure---reveal anything meaningful about the underlying distribution
from which the data is sampled. Understanding what statistical estimation based on
a finite data set reveals about the underlying distribution is a central preoccupation
of statistics and machine learning; however this kind of analysis has proved elusive
in the case of clustering, except perhaps in the case of density-based clustering.

Consider for instance $k$-means, possibly the most popular clustering
procedure in use today. If this procedure returns $k$ clusters on an $n$-sample
from a distribution $f$, what do these clusters reveal about $f$?
\cite{P81} proved a basic consistency result:
if the algorithm always finds the global minimum of the $k$-means cost function
(which, incidentally, is NP-hard and thus computationally intractable in general;
see \cite{DF09}, Theorem~3), then as $n \rightarrow \infty$,
the clustering is the globally optimal $k$-means solution for $f$, suitably defined.
Even then, it is unclear whether the best $k$-means solution to $f$ is
an interesting or desirable quantity in settings outside of vector quantization.

Our work, and more generally work on density-based clustering, relies on meaningful
formalisms of how a clustering of data generalizes to unambiguous structures of
the underlying distribution. The main such formalism is that of the cluster tree.

\subsection{The cluster tree}\label{ssec-clustertree}

We start with notions of connectivity. A {\em path} $P$ in $S \subset \X$
is a continuous function $P: [0,1] \rightarrow S$. If $x = P(0)$
and $y = P(1)$, we write $x \stackrel{P}{\leadsto} y$ and we say that
$x$ and $y$ are connected in $S$. This relation -- ``connected in $S$''
-- is an equivalence relation that partitions $S$ into its
{\it connected components}. We say $S \subset \X$ is {\em connected}
if it has a single connected component.

The cluster tree is a hierarchy each of whose levels is a {\it partition of a
subset} of $\X$, which we will occasionally call a {\it subpartition}
of $\X$. Write $\Pi(\X) = \{\mbox{subpartitions of $\X$}\}$.

\begin{defn}
For any $f: \X \rightarrow \R$, the {\em cluster tree of $f$} is a function
$\C_f: \R \rightarrow \Pi(\X)$ given by
$ \C_f(\lambda)
= \mbox{connected components of $\{x \in \X: f(x) \geq \lambda\}$} .$
Any element of $\C_f(\lambda)$, for any $\lambda$, is called a {\em cluster} of $f$.
\end{defn}
For any $\lambda$, $\C_f(\lambda)$ is a set of disjoint clusters of $\X$.
They form a hierarchy in the following sense.
\begin{lemma}
Pick any $\lambda' \leq \lambda$. Then:
\begin{enumerate}
\item For any $C \in \C_f(\lambda)$, there exists $C' \in \C_f(\lambda')$
such that $C \subseteq C'$.
\item For any $C \in \C_f(\lambda)$ and $C' \in \C_f(\lambda')$,
either $C \subseteq C'$ or $C \cap C' = \emptyset$.
\end{enumerate}
\label{lemma:hierarchy}
\end{lemma}
We will sometimes deal with the restriction of the cluster tree to a
finite set of points $x_1, \ldots, x_n$. Formally, the restriction
of a subpartition $\C \in \Pi(\X)$ to these points is defined to be
$ \C[x_1, \ldots, x_n] = \{C \cap \{x_1, \ldots, x_n\}: C \in \C\}$.
Likewise, the restriction of the cluster tree is
$\C_f[x_1, \ldots, x_n]: \R \rightarrow \Pi(\{x_1, \ldots, x_n\})$, where
$\C_f[x_1, \ldots, x_n](\lambda) = \C_f(\lambda)[x_1, \ldots, x_n] $
(Figure~\ref{fig:tree-example}).

\begin{figure}
\begin{center}
\resizebox{3.3in}{!}{\input{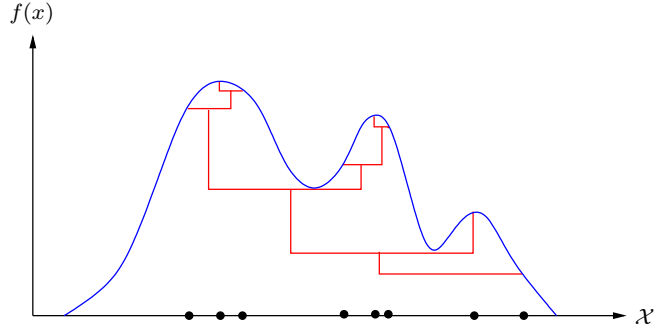}}
\end{center}
\caption{A probability density $f$, and the restriction of $\C_f$ to a finite set of
eight points.}
\label{fig:tree-example}
\end{figure}

\subsection{Notion of convergence and previous work}\label{ssec-related}

Suppose a sample $X_n \subset \X$ of size $n$ is used to construct a tree $\C_n$
that is an estimate of $\C_f$. \citet{H81} provided a sensible
notion of consistency for this setting.
\begin{defn}
For any sets $A,A' \subset \X$, let $A_n$ (resp, $A'_n$) denote the smallest
cluster of $\C_n$ containing $A \cap X_n$ (resp, $A' \cap X_n$). We say $\C_n$
is \emph{consistent} if, whenever $A$ and $A'$ are different connected components of
$\{x: f(x) \geq \lambda\}$ (for some $\lambda > 0$),
$\P(\mbox{$A_n$ is disjoint from $A_n'$}) \rightarrow 1$ as
$n \rightarrow \infty$.
\label{defn:consistency}
\end{defn}
It is well known that if $X_n$ is used to build a uniformly
consistent density estimate $f_n$ (that is,
$\sup_x |f_n(x) - f(x)| \rightarrow 0$), then the cluster tree $\C_{f_n}$
is consistent; see the appendix for details. The problem is that $\C_{f_n}$
is not easy to compute for typical density estimates $f_n$: imagine, for
instance, how one might go about trying to find level sets of a mixture
of Gaussians! \citet{WL83} have an efficient procedure that
tries to approximate $\C_{f_n}$ when $f_n$ is a $k$-nearest neighbor
density estimate, but they have not shown that it preserves the consistency
of $\C_{f_n}$.

On the other hand, there is a simple and elegant algorithm that is a plausible
estimator of the cluster tree: {\it single linkage} (or {\it Kruskal's algorithm}).
Given a data set $x_1, \ldots, x_n \in \R^d$, it operates as follows.
\begin{enumerate}
\item For each $i$, set $r_2(x_i)$ to the distance from
$x_i$ to its nearest neighbor.
\item As $r$ grows from $0$ to $\infty$:
\begin{enumerate}
\item Construct a graph $G_r$ with nodes $\{x_i: r_2(x_i) \leq r\}$.

Include edge $(x_i, x_j)$ if $\|x_i - x_j\| \leq r$.
\item Let $\C_n(r)$ be the connected components of $G_r$.
\end{enumerate}
\end{enumerate}
\citet{H81} has shown that single linkage is consistent in one dimension
(that is, for $d = 1$). But he also demonstrates, by a lovely reduction to continuum
percolation, that this consistency fails in higher dimension $d \geq 2$.
The problem is the requirement that $A \cap X_n \subset A_n$: by the time the
clusters are large enough that one of them contains all of $A$, there is a
reasonable chance that this cluster will be so big as to also contain part of $A'$.

With this insight, Hartigan defines a weaker notion of {\it fractional consistency},
under which $A_n$ (resp, $A_n'$) need not contain {\it all} of $A \cap X_n$ (resp,
$A' \cap X_n$), but merely a sizeable chunk of it -- and ought to be very close (at
distance $\rightarrow 0$ as $n \rightarrow \infty$) to the remainder. He then shows
that single linkage achieves this weaker consistency for any pair $A, A'$ for
which the ratio
$$\frac{\inf\{f(x): x \in A \cup A'\}}{\sup\{\inf\{f(x): x \in P\} : \mbox{paths $P$ from $A$ to $A'$}\}}$$
is sufficiently large. More recent work by \citet{P95} closes the gap and
shows fractional consistency whenever this ratio is $>1$.

A more robust version of single linkage has been proposed by \citet{W69}:
when connecting points at distance $r$ from each other, only consider points
that have at least $k$ neighbors within distance $r$ (for some $k > 2$).
Thus initially, when $r$ is small, only the regions of highest density are
available for linkage, while the rest of the data set is ignored. As $r$
gets larger, more and more of the data points become candidates for linkage.
This scheme is intuitively sensible, but Wishart does not provide a proof
of convergence. Thus it is unclear how to set $k$, for instance.

Several papers~\citep{RV09, MHL09, SSN09, RW10} have recently
considered the problem of recovering the connected components of
$\{x: f(x) \geq \lambda\}$ for a user-specified $\lambda$: the {\it flat}
version of our problem. Most similar to the work in this paper is the algorithm
of \citet{MHL09}, which uses the $k$-nearest neighbor graph of the data.
These level set results invariably require \emph{niceness} conditions
on the specific level set being recovered, often stated in terms of the
smoothness of the boundary of clusters \cite{}, and/or regularity conditions on the
density $f$ on clusters of the given level set. It is unclear whether these
conditions hold for all level sets of a general density, in other words how restrictive
these conditions are in the context of recovering the entire cluster tree.
In contrast, under mild requirements on the distribution,
our conditions on the recovered level sets hold for \emph{any} level set as the sample
size $n$ increases. The main distributional requirement for consistency is that of
continuity of the density $f$ on a compact support $\X$.

A different approach is taken in a paper of \citet{S11},
which does not require the user to specify a density level, but rather
automatically determines the smallest $\lambda$ at which $\C_f(\lambda)$
has two components. In \citet{S11} the continuity requirements on the density are milder
than for other results in the literature, including ours. However it
does restrict attention to bimodal densities due to technical hurdles of the flat case:
different levels of the cluster tree are collapsed together in the flat case making it difficult
to recover a given level from data especially in the case of multimodal densities.
Interestingly, the hierarchical
setting resolves some of the technical hurdles of the flat case since levels of the cluster tree
would generally appear at different levels of a sensible hierarchical estimator.
This makes it possible in this paper to give particularly simple estimators, and to analyze them under quite
modest assumptions on the data.

A related issue that has received quite a lot of attention is that of {\it pruning}
a cluster tree estimate: removing spurious clusters. A recent result of
\citet{RSNW12} gives meaningful statistical guarantees, but is based on the
cluster tree of an empirical density estimate, which is algorithmically
problematic as discussed  earlier. \citet{SN10} have an appealing top-down scheme
for estimating the cluster tree, along with a post-processing step (called
{\it runt pruning}) that helps identify modes of the distribution. The consistency
of this method has not yet been established. We provide a consistent
pruning procedure for both our procedures.

The present results are based in part on earlier conference versions,
namely \cite{CD10} and \cite{KV11}. The result of \cite{CD10} analyzes the
consistency of the first cluster tree estimator (see next section) but provides no
pruning method for the estimator. The result of \cite{KV11} analyzes the second
cluster tree estimator and shows how to prune it. However the pruning
method is tuned to this second estimator and works only under strict H\"{o}lder
continuity requirements on the density. The present work first provides a unified
analysis of both estimators using techniques developed in \cite{CD10}. Second,
building on insight from \cite{KV11}, we derive a new pruning method which
proveably works for either estimator without H\"{o}lder conditions on the distribution.
In particular, the pruned version of either cluster tree estimate
remains consistent under mild uniform continuity assumptions.
The main finite-sample pruning result of Theorem \ref{thm:pruning} requires even milder
conditions on the density than required for consistency.

\section{Algorithms and results}

The first algorithm we consider in this paper is a generalization of Wishart's
scheme and of single linkage, shown in Figure~\ref{fig:alg}. It has two free
parameters: $k$ and $\alpha$. For practical reasons, it is of interest to keep
these as small as possible.  We provide finite-sample convergence rates for all
$1 \leq \alpha \leq 2$ and we can achieve $k \sim d \log n$ if
$\alpha \geq \sqrt{2}$. Our rates for $\alpha = 1$ force $k$ to
be much larger, exponential in $d$. It is an open problem to determine whether
the setting ($\alpha = 1, k \sim d \log n$) yields consistency.

\begin{figure}
\framebox[6.5in]{
\begin{minipage}[t]{6in}
\noindent
{\bf Algorithm 1}

\begin{enumerate}
\item For each $x_i$ set
$r_k(x_i) = \min\{r: \mbox{$B(x_i,r)$ contains $k$ data points}\}$.
\item As $r$ grows from $0$ to $\infty$:
\begin{enumerate}
\item Construct a graph $G_r$ with nodes $\{x_i: r_k(x_i) \leq r\}$.

Include edge $(x_i, x_j)$ if $\|x_i - x_j\| \leq \alpha r$.
\item Let $\C_n(r)$ be the connected components of $G_r$.
\end{enumerate}
\end{enumerate}
\end{minipage}
}
\caption{An algorithm for hierarchical clustering. The input is a sample
$X_n = \{x_1, \ldots, x_n\}$ from density $f$ on $\X$. Parameters $k$
and $\alpha$ need to be set. Single linkage is $(\alpha=1,k=2)$.
Wishart suggested $\alpha = 1$ and larger $k$.}
\label{fig:alg}
\end{figure}

Conceptually, the algorithm creates a series of graphs $G_r = (V_r, E_r)$
satisfying a nesting property:
$
r \leq r' \ \ \ \Rightarrow \ \ \ V_r \subset V_{r'} \mbox{\ and \ } E_r \subset E_{r'}.
$
A point is admitted into $G_r$ only if it has $k$ neighbors within distance $r$;
when $r$ is small, this picks out the regions of highest density, roughly.
The edges of $G_r$ are between all pairs of points within distance $\alpha r$
of each other.

In practice, the only values of $r$ that matter are those corresponding to
interpoint distances within the sample, and thus the algorithm is efficient.
A further simplification is that the graphs $G_r$ don't need to be explicitly
created. Instead, the clusters can be generated directly using Kruskal's
algorithm, as is done for single linkage.

\begin{figure}
\framebox[6.5in]{
\begin{minipage}[t]{6in}
\noindent
{\bf Algorithm 2}

\begin{enumerate}
\item For each $x_i$ set
$r_k(x_i) = \min\{r: \mbox{$B(x_i,r)$ contains $k$ data points}\}$.
\item As $r$ grows from $0$ to $\infty$:
\begin{enumerate}
\item Construct a graph $G^{\nn}_r$ with nodes $\{x_i: r_k(x_i) \leq r\}$.

Include edge $(x_i, x_j)$ if:
$$
\begin{array}{ll}
\|x_i - x_j\| \leq \alpha \max(r_k(x_i), r_k(x_j)) & \mbox{$k$-NN graph} \\
\|x_i - x_j\| \leq \alpha \min(r_k(x_i), r_k(x_j)) & \mbox{mutual $k$-NN graph} \\
\end{array}
.
$$
\item Let $\C_n(r)$ be the connected components of $G^{\nn}_r$.
\end{enumerate}
\end{enumerate}
\end{minipage}
}
\caption{A cluster tree estimator based on the $k$-nearest neighbor graph.}
\label{fig:alg2}
\end{figure}

The second algorithm we study (Figure~\ref{fig:alg2}) is based on the $k$-nearest
neighbor graph of the samples. There are two natural ways to define this graph,
and we will analyze the sparser of the two, the mutual $k$-NN graph, which we
shall denote $G^{\nn}$. Our results hold equally for the other variant.

One way to think about the second hierarchical clustering algorithm is that
it creates the $k$-nearest neighbor graph on all the data samples, and then
generates a hierarchy by removing points from the graph in decreasing order
of their $k$-NN radius $r_k(x_i)$. The resulting graphs $G^{\nn}_r$ have the
same nodes as the corresponding $G_r$ but have potentially fewer edges:
$E^{\nn}_r \subset E_r$. This makes them more challenging to analyze.

Much of the literature on density-based clustering refers to clusters not
by the radius $r$ at which they appear, but by the ``corresponding empirical density'',
which in our case would be $\lambda = k/(n v_d r^d)$, where $v_d$ is the
volume of the unit ball in $\R^d$. The reader who is more comfortable with
the latter notation should mentally substitute $G[\lambda]$ whenever we refer
to $G_r$. We like using $r$ because it is directly observed rather than inferred.
Consider, for instance, a situation in which the underlying density $f$ is
supported on a low-dimensional submanifold of $\R^d$. The two cluster tree
algorithms continue to be perfectly sensible, as does $r$; but the inferred
$\lambda$ is misleading.

\subsection{A notion of cluster salience}

Suppose density $f$ is supported on some subset $\X$ of $\R^d$. We will find
that when Algorithms 1 and 2 are run on data drawn from $f$, their estimates
are consistent in the sense of Definition~\ref{defn:consistency}.
But an even more interesting question is, what clusters will be identified from a
{\it finite} sample? To answer this, we need a notion of salience.

The first consideration is that a cluster is hard to identify if it contains
a thin ``bridge'' that would make it look disconnected in a small sample.
To control this, we consider a ``buffer zone'' of width $\sigma$ around the
clusters.
\begin{defn}
For $Z \subset \R^d$ and $\sigma > 0$, write $Z_{\sigma} = Z + B(0,\sigma) =
\{y \in \R^d: \inf_{z\in Z} \|y-z\| \leq \sigma\}$.
\end{defn}
$Z_\sigma$ is a full-dimensional set, even if $Z$ itself is not.

Second, the ease of distinguishing two clusters $A$ and $A'$ depends inevitably
upon the separation between them. To keep things simple, we'll use the same
$\sigma$ as a separation parameter.
\begin{defn}
Let $f$ be a density supported on $\X \subset \R^d$. We say that $A, A' \subset \X$ are
{\bf$(\sigma,\epsilon)$-separated} if there exists $S \subset \X$
(\emph{separator set}) such that ({\rm i}) any path in $\X$ from $A$ to $A'$ intersects $S$,
and ({\rm ii}) $\sup_{x \in S_\sigma} f(x) < (1-\epsilon) \inf_{x \in A_\sigma \cup A'_\sigma} f(x)$.
\end{defn}
Under this definition, $A_\sigma$ and $A'_\sigma$ must lie within $\X$, otherwise
the right-hand side of the inequality is zero. $S_\sigma$ need not be contained in $\X$.

\subsection{Consistency and rate of convergence}

We start with a result for Algorithm 1, under the settings $\alpha \geq \sqrt{2}$
and $k \sim d \log n$. The analysis section also has results for $1 \leq \alpha \leq 2$
and $k \sim (2/\alpha)^d d \log n$. The result states general saliency conditions under
which a given level $\lambda$ of the cluster tree is recovered at level $r(\lambda)$
of the estimator. The mapping $r$ is of the form $\paren{\frac{k}{nv_d\lambda}}^{1/d}$ (see Definition \ref{defn:rlambda}),
where $v_d$ is the volume of the unit ball in $\R^d$.

\begin{thm}
There is an absolute constant $C$ such that the following holds. Pick any
$0 < \delta, \epsilon < 1$, and run Algorithm 1 on a sample $X_n$ of size $n$
drawn from $f$, with settings
$$ \sqrt{2} \leq \alpha \leq 2
\mbox{\ \ \ and\ \ \ } k \geq C \cdot \frac{d \log n}{\epsilon^2} \cdot \log^2 \frac{1}{\delta} .$$
Then there is a mapping $r: [0,\infty) \rightarrow [0, \infty)$ such that
the following holds with probability at least $1-\delta$.
Consider {any} pair of connected subsets $A,A' \subset \X$ such that
$A,A'$ are $(\sigma,\epsilon)$-separated for $\epsilon$ and some $\sigma > 0$.
Let $\lambda = \inf_{x \in A_\sigma \cup A'_\sigma} f(x)$. If
$
n
\ \geq \
\frac{k}{v_d (\sigma/2)^d \lambda} \left( 1 + \frac{\epsilon}{2} \right), \text{ then:}
$
\begin{enumerate}
\item {\it Separation.} $A \cap X_n$ is disconnected from $A' \cap X_n$
in $G_{r(\lambda)}$.
\item {\it Connectedness.} $A \cap X_n$ and $A' \cap X_n$ are each
connected in $G_{r(\lambda)}$.
\end{enumerate}
\label{thm:upper-bound}
\end{thm}
The two parts of this theorem -- separation and connectedness -- are proved
in Sections~\ref{sec:separation} and \ref{sec:connectedness}, respectively.

A similar result holds for Algorithm 2 under stronger requirements on $k$.
\begin{thm}
Theorem~\ref{thm:upper-bound} applies also to Algorithm 2, provided the following
additional condition on $k$ is met:
$
k \geq \frac{\Lambda}{\lambda} \cdot C d \log n \cdot \log \frac{1}{\delta},
$
where $\Lambda = \sup_{x \in \X} f(x)$.
\label{thm:upper-bound2}
\end{thm}
In the analysis section, we give a lower bound (Lemma~\ref{lem:lowerboundexample})
that shows why this dependence on $\Lambda/\lambda$ is needed.

Finally, we point out that these finite-sample results imply consistency
(Definition~\ref{defn:consistency}): as $n \rightarrow \infty$, take
$k_n = (d \log n)/\epsilon_n^2$ with any schedule of $\braces{\epsilon_n}$
such that $\epsilon_n \rightarrow 0$ and $k_n/n \rightarrow 0$.
Under mild uniform continuity conditions, any two connected components $A,A'$ of $\{f \geq \lambda\}$
are $(\sigma,\epsilon)$-separated for some $\sigma, \epsilon > 0$ (see appendix);
thus they are identified given large enough $n$.

\section{Analysis of Algorithm 1}

\subsection{Separation}
\label{sec:separation}

Both cluster tree algorithms depend heavily on the radii $r_k(x)$: the distance within
which $x$'s nearest $k$ neighbors lie (including $x$ itself). The empirical
probability mass of $B(x,r_k(x))$ is $k/n$. To show that $r_k(x)$ is meaningful,
we need to establish that the mass of this ball under density $f$ is also roughly
$k/n$. The uniform convergence of these empirical counts follows from
the fact that balls in $\R^d$ have finite VC dimension, $d+1$.

We also invoke uniform convergence over {\it half-balls}: each of these is
the intersection of a ball with a halfspace through its center.
Using uniform Bernstein-type bounds, we derive basic inequalities which we
use repeatedly.

\begin{lemma}
Assume $k \geq d \log n$, and fix some $\delta > 0$. Then there exists a constant
$C_\delta$ such that with probability $> 1-\delta$, we have that, first, every ball
$B \subset \R^d$ satisfies the following conditions:
\begin{eqnarray*}
f(B) \geq \frac{C_\delta d \log n}{n} & \implies & f_n(B) > 0 \\
f(B) \geq \frac{k}{n} + \frac{C_\delta}{n} \sqrt{kd \log n}
& \implies & f_n(B) \geq \frac{k}{n} \\
f(B) \leq \frac{k}{n} - \frac{C_\delta}{n} \sqrt{kd \log n}
& \implies & f_n(B) < \frac{k}{n}
\end{eqnarray*}
Here $f_n(B) = |X_n \cap B|/n$ is the empirical mass of $B$, while
$f(B)$ is its probability under $f$.
Second, for every half-ball $H \subset \R^d$:
\begin{eqnarray*}
f(H) \geq \frac{C_\delta d \log n}{n} & \implies & f_n(H) > 0.
\end{eqnarray*}
We denote this uniform convergence over balls and half-balls as
event $E_o$.
\label{lemma:convergence-balls}
\end{lemma}
\begin{proof}
See appendix. $C_\delta = 2C_o \log (2/\delta)$,
where $C_o$ is the absolute constant from Lemma~\ref{lemma:convergence-balls1}.
\end{proof}
We will typically preface other results by a statement like ``Assume $E_o$.''
It is to be understood that $E_o$ occurs with probability at least $1-\delta$
over the random sample $X_n$, where $\delta$ is henceforth fixed. The constant
$C_\delta$ will keep reappearing through the paper.

For any cluster $A \subset \X$, there is a certain scale $r$ at which every
data point in $A$ appears in $G_r$. What is this $r$?
\begin{lemma}
Assume $E_o$. Pick any set $A \subset \X$, and let
$\lambda = \inf_{x \in A_\sigma} f(x)$. If $r < \sigma$ and
$
v_d r^d \lambda \ \geq \  \frac{k}{n} + \frac{C_\delta}{n} \sqrt{kd \log n},
$
then $G_r$ contains every point in $A_{\sigma - r} \cap X_n$.
\label{lemma:active-points}
\end{lemma}
\begin{proof}
Any point $x \in A_{\sigma-r}$ has $f(B(x,r)) \geq v_d r^d \lambda$; and thus,
by Lemma~\ref{lemma:convergence-balls}, has at least $k$ neighbors within radius $r$.
\end{proof}

In order to show that two separate clusters $A$ and $A'$ get distinguished in the
cluster tree, we need to exhibit a scale $r$ at which every point in $A$ and $A'$
is active, but there is no path from $A$ to $A'$.

\begin{lemma}
Assume $E_o$. Suppose sets $A, A' \subset \X$ are $(\sigma, \epsilon)$-separated
by set $S$, and let $\lambda = \inf_{x \in A_\sigma \cup A'_\sigma} f(x)$. Pick
$0 < r < \sigma$ such that
\begin{equation*}
 \frac{k}{n} + \frac{C_\delta}{n} \sqrt{kd \log n} \leq  v_d r^d \lambda <
 \paren{\frac{k}{n} - \frac{C_\delta}{n} \sqrt{kd \log n}}\cdot\frac{1}{1 - \epsilon}.
\end{equation*}
 Then:
\begin{enumerate}
\item[(a)] $G_r$ contains all points in $(A_{\sigma-r} \cup A'_{\sigma-r}) \cap X_n$.
\item[(b)] $G_r$ contains no points in $S_{\sigma-r} \cap X_n$.
\item[(c)] If $r < 2\sigma/(\alpha + 2)$, then $A \cap X_n$ is disconnected from
$A' \cap X_n$ in $G_r$.
\end{enumerate}
\label{lemma:separation}
\end{lemma}
\begin{proof}
Part (a) is directly from Lemma~\ref{lemma:active-points}.
For (b), any point $x \in S_{\sigma-r}$ has $f(B(x,r)) < v_d r^d \lambda(1-\epsilon)$;
and thus, by Lemma~\ref{lemma:convergence-balls},
has strictly fewer than $k$ neighbors within distance $r$.

For (c), since points in $S_{\sigma-r}$ are absent from $G_r$, any path from
$A$ to $A'$ in that graph must have an edge across $S_{\sigma-r}$. But any such
edge has length at least $2(\sigma-r) > \alpha r$ and is thus not in $G_r$.
\end{proof}

\begin{defn}
Define $r(\lambda)$ to be the value of $r$ for which
$ v_d r^d \lambda = \frac{k}{n} + \frac{C_\delta}{n} \sqrt{kd \log n} $.
\label{defn:rlambda}
\end{defn}

\begin{cor}
\label{cor:separation}
The conditions of Lemma~\ref{lemma:separation} are satisfied by $r = r(\lambda)$
if $r(\lambda) < 2\sigma/(\alpha + 2)$ and $k \geq 4C_\delta^2 (d/\epsilon^2) \log n$.
\end{cor}

\subsection{Connectedness}
\label{sec:connectedness}

We need to show that points in $A$ (and similarly $A'$) are connected in
$G_{r(\lambda)}$. First we state a simple bound (proved in the appendix)
that works if $\alpha = 2$ and $k \sim d \log n$; later we consider smaller $\alpha$.
\begin{lemma}
Assume $E_o$. Let $A$ be a connected set in $\X$ with $\lambda = \inf_{x \in A_\sigma} f(x)$.
Suppose $1 \leq \alpha \leq 2$. Then $A \cap X_n$ is connected in $G_r$ whenever
$r \leq  2 \sigma/(2 + \alpha)$ and
$$ v_d r^d \lambda
\ \ \geq \ \
\max\left\{ \left( \frac{2}{\alpha} \right)^d \frac{C_\delta d \log n}{n},
\ \ \frac{k}{n} + \frac{C_\delta}{n} \sqrt{kd \log n}
\right\}.
$$
\label{lemma:connectedness}
\end{lemma}
Comparing this to the definition of $r(\lambda)$, we see that choosing
$\alpha = 1$ would entail $k \geq 2^d$, which is undesirable. We can get
a more reasonable setting of $k \sim d \log n$ by choosing $\alpha = 2$,
but we'd like $\alpha$ to be as small as possible. A more refined argument
shows that $\alpha \approx \sqrt{2}$ is enough.

\begin{thm}
Assume $E_o$. Let $A$ be a connected set in $\X$ with $\lambda = \inf_{x \in A_\sigma} f(x)$.
Suppose $\alpha \geq \sqrt{2}$.
Then $A \cap X_n$ is connected in $G_r$ whenever $r \leq \sigma/2$ and
$$ v_d r^d \lambda
\ \ \geq \ \
\max\left\{ \frac{4 C_\delta d \log n}{n},
\ \ \frac{k}{n} + \frac{C_\delta}{n} \sqrt{kd \log n}
\right\}.
$$
\label{thm:connectedness}
\end{thm}
\begin{proof}
Recall that a {\it half-ball} is the intersection of an open ball and a halfspace
through the center of the ball. Formally, it is defined by a center $\mu$, a radius
$r$, and a unit direction $u$:
$$ \{z \in \R^d: \|z - \mu\| < r, (z - \mu) \cdot u > 0 \} .$$
We will describe any such set as ``the half of $B(\mu,r)$ in direction $u$''.
If the half-ball lies entirely in $A_\sigma$, its probability mass is at least
$(1/2) v_d r^d \lambda$. By uniform convergence bounds
(Lemma~\ref{lemma:convergence-balls}), if $v_d r^d \lambda \geq (4 C_\delta d \log n)/n$,
then every such half-ball within $A_\sigma$ contains at least one data point.

Pick any $x, x' \in A \cap X_n$; there is a path $P$ in $A$ with
$x \stackrel{P}{\leadsto} x'$.  We'll identify a sequence of
data points $x_0 = x, x_1, x_2, \ldots$, ending in $x'$, such that
for every $i$, point $x_i$ is active in $G_r$ and $\|x_i - x_{i+1}\| \leq \alpha r$.
This will confirm that $x$ is connected to $x'$ in $G_r$.
\begin{figure}
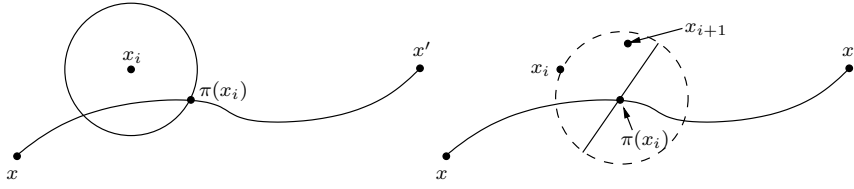

\begin{center}
{\resizebox{2.2in}{!}{\input{path1.pstex_t}}}
{\resizebox{2.2in}{!}{\input{path3.pstex_t}}}
\end{center}
\caption{{\it Left:} $P$ is a path from $x$ to $x'$, and $\pi(x_i)$ is the point
furthest along the path that is within distance $r$ of $x_i$.
{\it Right:} The next point, $x_{i+1} \in X_n$, is chosen from the half-ball of
$B(\pi(x_i), r)$ in the direction of $x_i - \pi(x_i)$.}
\label{fig:construction}
\end{figure}

To begin with, recall that $P$ is a continuous function from $[0,1]$
into $A$. For any point $y \in \X$, define $N(y)$ to be the portion of $[0,1]$
whose image under $P$ lies in $B(y,r)$: that is,
$N(y) = \{0 \leq z \leq 1: P(z) \in B(y,r) \}$. If $y$ is within distance
$r$ of $P$, then $N(y)$ is nonempty. Define $\pi(y) = P(\sup N(y))$, the
furthest point along the path within distance
$r$ of $y$ (Figure~\ref{fig:construction}, left).

The sequence $\braces{x_i}$ is defined iteratively; $x_0 = x$,
and for $i = 0, 1, 2, \ldots:$
\begin{itemize}
\item If $\|x_i - x'\| \leq \alpha r$, set $x_{i+1} = x'$ and stop.
\item By construction, $x_i$ is within distance $r$ of path $P$ and hence $N(x_i)\neq \emptyset$.
\item Let $B$ be the open ball of radius $r$ around $\pi(x_i)$. The half of $B$
in direction $x_i - \pi(x_i)$ contains a data point; this is $x_{i+1}$
(Figure~\ref{fig:construction}, right).
\end{itemize}

The process eventually stops since each $\pi(x_{i+1})$ is further
along path $P$ than $\pi(x_i)$; formally, $\sup N(x_{i+1}) > \sup N(x_i)$.
This is because $\|x_{i+1} - \pi(x_i)\| < r$, so by continuity of the function $P$,
there are points further along $P$ (beyond $\pi(x_i)$) whose distance
to $x_{i+1}$ is still $< r$. Thus $x_{i+1}$ is distinct from $x_0, x_1, \ldots, x_i$.
Since there are finitely many data points, the process must terminate, so
the sequence $\braces{x_i}$ constitutes a path from $x$ to $x'$.

Each $x_i$ lies in $A_r \subseteq A_{\sigma-r}$ and is thus active in $G_r$
under event $E_o$ (Lemma~\ref{lemma:active-points}). Finally, the distance
between successive points is $\lefteqn{\|x_i - x_{i+1}\|^2 }$
\begin{eqnarray*}
& = &
\|x_i - \pi(x_i) + \pi(x_i) - x_{i+1}\|^2 \\
& = &
\|x_i - \pi(x_i)\|^2 + \|\pi(x_i) - x_{i+1}\|^2
- 2 (x_i - \pi(x_i)) \cdot (x_{i+1} - \pi(x_i))
\\
& \leq &
2r^2 \ \ \leq \ \ \alpha^2 r^2,
\end{eqnarray*}
where the second-last inequality is from the definition of half-ball.
\end{proof}

To complete the proof of Theorem~\ref{thm:upper-bound}, take
$k \geq 4C_\delta^2 (d/\epsilon^2) \log n$.
The relationship that defines $r = r(\lambda)$ (Definition~\ref{defn:rlambda})
then implies
$$ \frac{k}{n}
\ \leq \ v_d r^d \lambda \ \leq \
\frac{k}{n} \left( 1 + \frac{\epsilon}{2} \right) .$$
This shows that clusters at density level $\lambda$ emerge when the growing radius
$r$ of the cluster tree algorithm reaches roughly $(k/(\lambda v_d n))^{1/d}$.
In order for $(\sigma, \epsilon)$-separated clusters to be distinguished,
the one additional requirement of Lemma~\ref{lemma:separation} and
Theorem~\ref{thm:connectedness} is that $r = r(\lambda)$ be at most $\sigma/2$;
this is what yields the final lower bound on $n$.

\section{Analysis of Algorithm 2}

The second cluster tree estimator (Figure~\ref{fig:alg2}), based on the
$k$-nearest neighbor graph of the data points, satisfies the same
guarantees as the first, under a more generous setting of $k$.

Let $G^{\nn}_r$ be the $k$-NN graph at radius $r$. We have already observed
that $G^{\nn}_r$ has the same vertices as $G_r$, and a subset of its edges.
Therefore, if clusters are separated in $G_r$, they are certainly separated
in $G^{\nn}_r$: the separation properties of Lemma~\ref{lemma:separation}
carry over immediately to the new estimator. What remains is to establish
a connectedness property, an analogue of Theorem~\ref{thm:connectedness},
for these potentially much sparser graphs.

\subsection{Connectivity properties}

As before, let $f$ be a density on $\X \subset \R^d$.
Let $\Lambda = \sup_{x \in \X} f(x)$; then the smallest radius we expect to
be dealing with is roughly $(k/(n v_d \Lambda))^{1/d}$. To be safe, let's pick
a value slightly smaller than this, and define $r_o = (k/(2n v_d \Lambda))^{1/d}$.

We'll first confirm that $r_o$ is, indeed, a lower bound on the radii $r_k(\cdot)$.
\begin{lemma}
Assume $E_o$. If $k \geq 4C_{\delta}^2 d \log n$, then $r_k(x) > r_o$ for all $x$.
\label{lem:concentration2}
\end{lemma}

\begin{proof}
Pick any $x$ and consider the ball $B(x, r_o)$. By definition of $r_o$,
$$
 f(B(x,r_o))
\ \leq \
v_d r_o^d \Lambda
\ = \
\frac{k}{2n}
\ \leq \
\frac{k}{n} - \frac{C_{\delta}}{n} \sqrt{k d \log n}
$$
where the last inequality is from the condition on $k$.
Under $E_o$ (Lemma~\ref{lemma:convergence-balls}), we then get
$f_n(B(x,r_o)) < k/n$; therefore $r_k(x) > r_o$.
\end{proof}

Now we present an analogue of Theorem~\ref{thm:connectedness}.
\begin{thm}
Assume $E_o$. Let $A$ be a connected set in $\X$, with $\lambda = \inf_{x \in A_\sigma} f(x)$.
Suppose $\alpha \geq \sqrt{2}$.
Then $A \cap X_n$ is connected in $G^{\nn}_r$ whenever
$r + r_o \leq \sigma$ and
$$ v_d r^d \lambda
\ \ \geq \ \
\frac{k}{n} + \frac{C_\delta}{n} \sqrt{kd \log n}
$$
and
$$ k \ \ \geq \ \
\max\left\{\frac{\Lambda}{\lambda} \cdot 8 C_\delta d \log n, 4 C_\delta^2 d \log n \right\}.$$
\label{thm:connectedness2}
\end{thm}

\begin{proof}
We'll consider events at two different scales: a small radius $r_o$, and the
potentially larger radius $r$ from the theorem statement.

Let's start with the small scale. The lower bound on $k$ yields
$$ v_d r_o^d \lambda
\ = \
\frac{k}{2n} \cdot \frac{\lambda}{\Lambda}
\ \geq \ \frac{4 C_\delta d \log n}{n} .$$
As in Theorem~\ref{thm:connectedness}, this implies that every half-ball of
radius $r_o$ within $A_\sigma$ contains at least one data point.

Let $x$ and $x'$ be any two points in $A \cap X_n$. As in
Theorem~\ref{thm:connectedness}, we can find a finite sequence of data points
$x = x_0, x_1, \ldots, x_p = x'$ such that for each $i$, two key conditions hold:
({\rm i}) $\|x_i - x_{i+1}\| \leq \alpha r_o$ and
({\rm ii}) $x_i$ lies within distance $r_o$ of $A$.

Now let's move to a different scale $r \leq \sigma - r_o$. Since each $x_i$
lies in $A_{r_o} \subseteq A_{\sigma-r}$, we know from
Lemma~\ref{lemma:active-points} that all $x_i$ are active in $G^{\nn}_r$ given
the lower bound on $v_d r^d \lambda$.
The edges $(x_i, x_{i+1})$ are also present, because
$$ \|x_i - x_{i+1}\| \leq \alpha r_o \leq \alpha \min(r_k(x_i), r_k(x_{i+1}))$$
using Lemma~\ref{lem:concentration2} and the bound on $k$.
Hence $x$ is connected to $x'$ in $G^{\nn}_r$.
\end{proof}
It is straightforward to check that $r(\lambda)$ is always $\geq r_o$, and
Theorem~\ref{thm:upper-bound2} follows immediately.

\subsection{A lower bound on neighborhood cardinality}

The result for $k$-nearest neighbor graphs requires a larger setting of $k$ than
our earlier result; in particular, $k$ needs to exceed the ratio $\Lambda/\lambda$.
We now show that this isn't just a looseness in our bound, but in fact a necessary
condition for these types of graphs.

Recall that the mutual $k$-NN graph contains all the data points, and puts an
edge between points $x$ and $x'$ if
$ \|x - x' \| \leq \alpha \min(r_k(x), r_k(x')) $
(the $\alpha$ is our adaptation). We will assume $1 \leq \alpha \leq 2$, as
is the case in all our upper bounds.

\begin{lemma}
Pick any $\lambda > 0$, any $\Lambda > 32 \lambda$, and any
$k \leq \Lambda/(64 \lambda)$. Then there is a density $f$ on $\X \subset \R$
with $\lambda \leq f(x) \leq \Lambda$ for all $x \in \X$, and with the following
property: for large enough $n$, when $n$ samples are drawn
i.i.d.\ from $f$, the resulting mutual $k$-NN graph (with $1 \leq \alpha \leq 2$)
is disconnected with probability at least $1/2$.
\label{lem:lowerboundexample}
\end{lemma}

\begin{figure}
\begin{center}
\resizebox{3in}{!}{\input{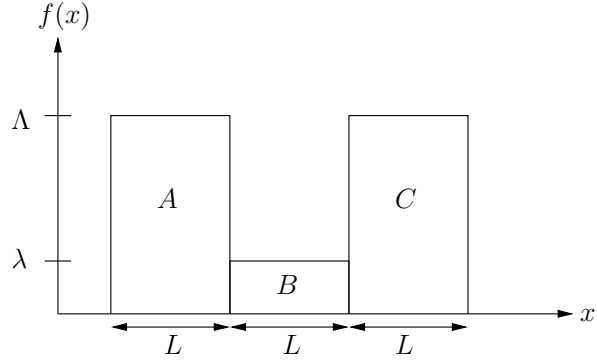}}
\end{center}
\caption{A density that illustrates why $k$-NN graphs require
$k \geq \Lambda/\lambda$  for connectivity.}
\label{fig:lbnd}
\end{figure}

\begin{proof}
Consider the density shown in Figure~\ref{fig:lbnd}, consisting of two dense
regions, $A$ and $C$, bridged by a less dense region $B$. Each region is of width
$L = 1/(\lambda + 2\Lambda)$. We'll show that the mutual $k$-NN graph of a sample
from this distribution is likely to be disconnected. Specifically, with
probability at least $1/2$, there will be no edges between $A$ and $B \cup C$.

To this end, fix any $n \geq \Lambda/\lambda$, and define
$\Delta = 1/(4n\lambda) < L$. Consider the leftmost portion of $B$ of
length $\Delta$. The probability that a random draw from $f$ falls in this
region is $\Delta \lambda = 1/(4n)$. Therefore, the probability that no
point falls in this region is $(1 - 1/(4n))^n \geq 3/4$. Call this event $E_1$.

Next, divide $A$ into intervals of length $\Delta, 2\Delta, 4 \Delta$, and so
on, starting from the right. We'll show that with probability at least $3/4$,
the right half of each such interval contains at least $k+1$ points; call this
event $E_2$. To see why, let's focus on one particular interval, say that
of length $2^i \Delta$.
The probability that a random point falls in the right half of this interval is
$2^{i-1} \Delta \Lambda \geq 2^{i+3} k/n$. Therefore, the number of points in this
region is $\geq 2^{i+3} k$ in expectation, and by a Chernoff bound, is $\geq k+1$
except with probability $< \exp(-2^{i+1}k)$. Taking a union bound over all
the intervals yields an overall failure probability of at most $1/4$.

With probability at least $1/2$, events $E_1$ and $E_2$ both occur. Whereupon,
for any point in $A$, its nearest neighbor in $B \cup C$ is at least twice as
far as its $k$ nearest neighbors in $A$. Thus the mutual $k$-NN graph has no
edges between $A$ and $B \cup C$.
\end{proof}
This constraint on $k$ is unpleasant, and it would be interesting to
either find mild smoothness assumptions on $f$, or better, modified
notions of $k$-NN graph, that render it unnecessary.

\section{Lower bound}

We have shown that the two cluster tree algorithms distinguish pairs
of clusters that are $(\sigma,\epsilon)$-separated. The number of samples
required to capture clusters at density $\geq \lambda$ is, by
Theorem~\ref{thm:upper-bound},
$$
O \left( \frac{d}{v_d (\sigma/2)^d \lambda \epsilon^2}
\log \frac{d}{v_d (\sigma/2)^d \lambda \epsilon^2} \right) ,
$$
We'll now show that this dependence on $\sigma$, $\lambda$, and $\epsilon$
is optimal. The only room for improvement, therefore, is in constants
involving $d$.

\begin{thm}
Pick any $0 < \epsilon < 1/2$, any $d > 1$, and any $\sigma, \lambda > 0$
such that $\lambda v_{d-1} \sigma^d < 1/120$. Then there exist: an input space
$\X \subset \R^d$; a finite family of densities $F = \{f_i\}$ on $\X$;
subsets $A_i, A_i', S_i \subset \X$ such that $A_i$ and $A_i'$ are
$(\sigma,\epsilon)$-separated by $S_i$ for density $f_i$, and
$\inf_{x \in A_{i,\sigma} \cup A'_{i,\sigma}} f_i(x) \geq \lambda$,
with the following additional property.

Consider any algorithm that is given $n \geq 100$ i.i.d.\ samples $X_n$ from
some $f_i \in F$ and, with probability at least $3/4$, outputs a
tree in which the smallest cluster containing $A_i \cap X_n$ is disjoint from
the smallest cluster containing $A_i' \cap X_n$. Then
$$ n \ \geq \
\frac{C_2}{v_d \sigma^d \lambda \epsilon^2 d^{1/2}} \log \frac{1}{v_d \sigma^d \lambda d^{1/2}} $$
for some absolute constant $C_2$.
\label{thm:lower-bound}
\end{thm}
\begin{proof}
Given the parameters $d, \sigma, \epsilon, \lambda$,
we will construct a space $\X$ and a finite family of densities
$F = \{f_i\}$ on $\X$.  We will then argue that any cluster tree
algorithm that is able to distinguish $(\sigma,\epsilon)$-separated clusters
must be able, when given samples from some $f_i$, to determine the
identity of $I$. The sample complexity of this latter task can be
lower-bounded using Fano's inequality (Appendix~\ref{sec:fano}): it is
$\Omega((\log |F|)/\theta)$, for
$$\theta = \max_{i \neq j} K(f_i, f_j),$$
where $K(\cdot,\cdot)$ is Kullback-Leibler divergence.

{\it The support $\X$.} The support $\X$ is made up of two disjoint regions:
a cylinder $\X_0$, and an
additional region $\X_1$ which serves as a repository for excess
probability mass. $\X_1$ can be chosen as any Borel set disjoint from
$\X_0$. The main region of interest is $\X_0$ and is described as follows
in terms of a constant $c>1$ to be specified. Pick
$1<\tau< \min\braces{\sigma, \paren{2c\cdot v_d \lambda}^{-1/d}} + 1$, such that
$\tau^{d-1}\leq 2$.
Let $B_{d-1}$ be
the unit ball in $\R^{d-1}$, and let $\tau\sigma B_{d-1}$ be this same ball scaled to
have radius $\tau\sigma$. The cylinder $\X_0$ stretches along the $x_1$-axis; its
cross-section is $\tau\sigma B_{d-1}$ and its length is $4(c+1)\sigma$ for some
$c > 1$ to be specified: $\X_0 = [0,4(c+1)\sigma] \times \sigma B_{d-1}$.
Here is a picture of it:

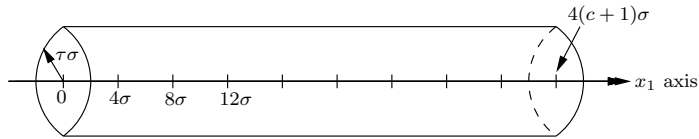
\begin{center}
\resizebox{3.3in}{!}{\input{cylinder1.pstex_t}}
\end{center}

{\it A family of densities on $\X$.} The family $F$ contains $c-1$
densities $f_1, \ldots, f_{c-1}$ which coincide on most of the
support $\X$ and differ on parts of $\X_0$. Each density
$f_i$ is piecewise constant as described below.
Items ({\rm ii}) and ({\rm iv}) describe the pieces
that are common to all densities in $F$.

\begin{enumerate}[(i)]
\item Density $\lambda(1-\epsilon)$ on
$(4 \sigma i + \sigma, 4\sigma i + 3 \sigma) \times \tau\sigma B_{d-1}$.
\item Balls of mass $1/(2c)$ centered at locations
$4\sigma, 8\sigma,\ldots, 4c \sigma$ along the $x_1$-axis:
each such ball is of radius $\tau -1<\sigma$, and the density
on these balls is $1/(2c\cdot v_d (\tau -1)^d)\geq\lambda$.
We refer to these as \emph{mass balls}.

\item Density $\lambda$ on the remainder of $\X_0$: this is the union
of the cylinder segments
$[0,4 \sigma i + \sigma] \times \tau\sigma B_{d-1}$ and $[4\sigma i + 3 \sigma, 4(c+1) \sigma] \times \tau\sigma B_{d-1}$
minus the mass balls.
Since the cross-sectional area of the cylinder is $v_{d-1} (\tau\sigma)^{d-1}$, the total
mass here is at most $\lambda \tau^{d-1}v_{d-1} \sigma^d (4(c+1) - 2)$.

\item The remaining mass is at least $1/2 - 8\lambda v_{d-1} \sigma^d (c+1)$;
we will be careful to choose $c$ so that this is nonnegative. The remaining mass is placed
on $\X_1$ in some fixed manner that does not vary between densities in $F$.
\end{enumerate}

Here is a sketch of $f_i$. The low-density region of width $2\sigma$ is centered
at $4 \sigma i + 2\sigma$ on the $x_1$-axis, and contains no mass balls.

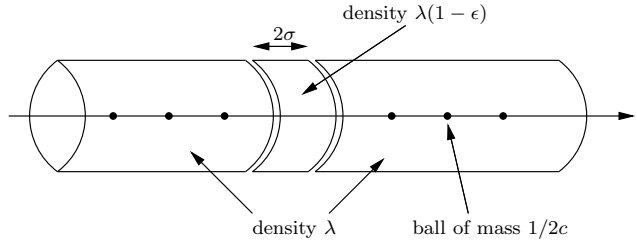
\begin{center}
\resizebox{3.3in}{!}{\input{cylinder2.pstex_t}}
\end{center}

For any $i \neq j$, the densities $f_i$ and $f_j$ differ only on the
cylindrical sections $(4 \sigma i + \sigma, 4\sigma i + 3 \sigma) \times \sigma B_{d-1}$
and $(4 \sigma j + \sigma, 4\sigma j + 3 \sigma) \times \sigma B_{d-1}$, which are
disjoint, contain no mass ball, and each have volume $2\tau^{d-1}v_{d-1} \sigma^d$. Thus
\begin{eqnarray*}
K(f_i, f_j)
& = &
2 \tau^{d-1}v_{d-1} \sigma^d \left( \lambda \log \frac{\lambda}{\lambda(1-\epsilon)} +
\lambda(1-\epsilon) \log \frac{\lambda(1-\epsilon)}{\lambda} \right)
\\
& = &
2 \tau^{d-1}v_{d-1} \sigma^d \lambda (-\epsilon \log (1-\epsilon))
\ \ \leq \ \
\frac{8}{\ln 2} v_{d-1} \sigma^d \lambda \epsilon^2
\end{eqnarray*}
(using $\ln(1-x) \geq -2x$ for $0 < x \leq 1/2$).
This is an upper bound on the $\theta$ in the Fano bound.

{\it Clusters and separators.}
Now define the clusters and separators as follows: for each $1 \leq i \leq c-1$,
\begin{itemize}
\item $A_i$ is the tubular segment $[\sigma, 4\sigma i]\times (\tau -1)\sigma$,
\item $A_i'$ is the tubular segment $[4 \sigma (i+1), 4(c+1)\sigma - \sigma]\times(\tau -1)\sigma$, and
\item $S_i = \{4\sigma i + 2\sigma\} \times \sigma B_{d-1}$ is the cross-section
of the cylinder at location $4 \sigma i + 2\sigma$.
\end{itemize}
Thus $A_i$ and $A_i'$ are $d$-dimensional sets while $S_i$ is a $(d-1)$-dimensional
set. It can be seen that, for density $f_i$, $A_i$ and $A_i'$ are $(\sigma,\epsilon)$-separated,
and $\inf_{x \in A_{i,\sigma} \cup A'_{i,\sigma}} f_i(x) \geq \lambda$.

Now that the various structures are defined, we still need to argue that if an
algorithm is given a sample $X_n$ from some $f_i$ (where $i$ is unknown),
and is able to separate $A_i \cap X_n$ from $A'_i \cap X_n$, then it can
effectively infer the identity of $i$. This has sample complexity
$\Omega((\log c)/\theta)$.

Let's set $c$ to be a small constant, say $c = 6$. Then, even a small sample
$X_n$ of $n \geq 100$ points is likely (with probability at least $3/4$, say),
to contain points from all of the $c$ mass balls, each of which has
mass $1/(2c)$. Suppose the algorithm even knows in advance that the underlying
density is one of the $c-1$ choices in $F$, and is subsequently able (with
probability at least $3/4$) to separate $A_i$ from $A_i'$. To do this, it must
connect all the points from mass balls within $A_i$, and all the points from mass balls within $A_i'$,
and yet keep these two groups apart. In short, this algorithm must be able to
determine (with overall probability at least $1/2$) the segment
$(4 \sigma i + \sigma, 4\sigma i + 3 \sigma)$ of lower density, and hence
the identity of $i$.

We can thus apply Fano's inequality to conclude that we need
$$ n
\ > \
\frac{\frac{1}{2} \log (c-1) - 1 }{\theta}
\ \geq \
\frac{(\frac{1}{2} \log 5  - 1) \ln 2}{8 v_{d-1} \sigma^d \lambda \epsilon^2}
\ \geq \
\frac{C_2}{v_d \sigma^d \lambda \epsilon^2 d^{1/2}}
$$
for some absolute constant $C_2$. The last equality comes from the formula
$v_d = \pi^{d/2}/\Gamma((d/2) + 1)$, whereupon $v_{d-1} = O(v_d d^{1/2})$.

This is almost the bound in the theorem statement, short  a logarithmic term.
To finish up, we now switch to a larger value of $c$:
$$ c \ = \ \left\lfloor \frac{1}{16 v_{d-1} \sigma^d \lambda}  - 1 \right\rfloor,$$
and apply the same construction. We have already established that we need
$n = \Omega(c/\epsilon^2)$ samples, so assume $n$ is at least this large. Then,
for small enough $\epsilon$, it is very likely that when the underlying density
is $f_i$, the sample $X_n$ will contain the four point masses at
$4\sigma$, $4\sigma i$, $4 \sigma (i+1)$, and $4 (c+1) \sigma$. Therefore, the
clustering algorithm must connect the point at $4 \sigma$ to that at $4 \sigma i$
and the point at $4 \sigma (i+1)$ to that at $4 (c+1) \sigma$, while keeping the
two groups apart. Therefore, this algorithm can determine $i$. Applying Fano's
inequality gives $n = \Omega((\log c)/\theta)$, which is the bound in the theorem
statement.
\end{proof}

\section{Pruning}

\begin{figure}
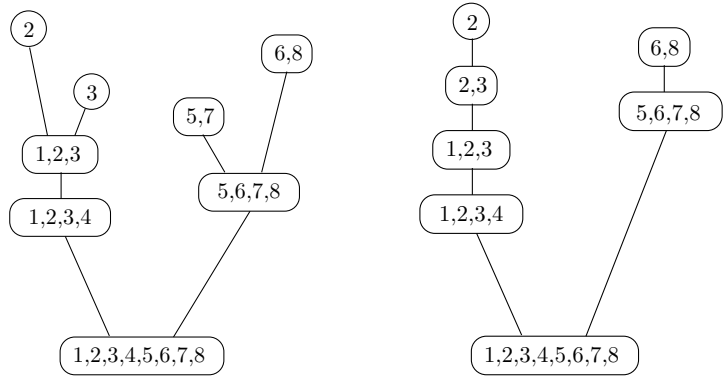

\begin{center}
\resizebox{1.6in}{!}{\input{pruning1.pstex_t}}
\hskip.5in
\resizebox{1.6in}{!}{\input{pruning2.pstex_t}}
\end{center}
\caption{A sample from the density of Figure~\ref{fig:lbnd} contains four
points ($1,2,3,4$) from cluster $A$ and four ($5,6,7,8$) from cluster $C$.
{\it Left:} A cluster tree that correctly distinguishes $A$ from $C$.
{\it Right:} A better alternative that avoids fragmenting $A$ and $C$.}
\label{fig:pruning}
\end{figure}

Hartigan's notion of consistency (Definition~\ref{defn:consistency})
requires distinct clusters to be distinguished, but does not guard against
fragmentation within a cluster. Consider, for instance, the density shown
in Figure~\ref{fig:lbnd}. Under Hartigan-consistency, in the limit, the cluster
tree must include a cluster that contains all of $A$ and a separate, disjoint
cluster that contains all of $C$. But the tree is allowed to break $A$ into
further subregions. To be concrete, suppose we draw a sample from that density
and receive four points from each of $A$ and $C$. Figure~\ref{fig:pruning},
left, shows a possible cluster tree on these samples that meets the consistency
requirement. However, we'd prefer the one on the right. Formally we want
to avoid or remove \emph{false clusters} as defined below.

\begin{defn}
 Let $A_n$ and $A_n'$ be the vertices of two separate connected components
 (potentially at different levels) in the cluster tree returned by an algorithm. We call
 $A_n$ and $A_n'$ \emph{false clusters} if they are part of the same connected
 component of the level set $\braces{x: f(x)\geq \min_{x'\in A_n \cup A_n'} f(x')}$.
\end{defn}

This problem is generally addressed in the literature by making assumptions
about the \emph{size} of true clusters. Real clusters are assumed to be large
in some sense, for instance in terms of their mass~\citep{MHL09}, or
\emph{excess mass}\footnote{The excess mass of a component $A$ at level
$\lambda$ is generally defined as $\int_A (f(x) -\lambda) \, d x$.} \citep{SN09}.
However, relying on size can be misleading in practice, as is illustrated in
Figure~\ref{fig:falseClusters}. It turns out that, building on the results
of the previous sections, there is a simple way to treat spurious clusters
independent of their size.

\begin{figure}
\begin{center}
\resizebox{6cm}{!}{\input{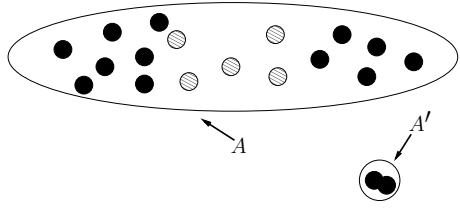}}
\end{center}
\caption{Depicted are samples from two connected components $A$ and $A'$ of some level set.
Suppose that only the black samples appear at level $r$ in the empirical tree: the grey
samples have $r_k(x)>r$. Then in $G_r$, samples from $A$ appear as two large clusters. If we
were to just go by size, we would be tempted to preserve this spurious partition and possibly
to remove the smaller cluster of samples from $A'$. Our pruning method, however, will connect
the two groups from $A$ and maintain the cluster from $A'$.}
\label{fig:falseClusters}
\end{figure}

\subsection{Intuition}

The pruning procedure of Figure~\ref{fig:alg_pruning} consists of a simple lookup:
it reconnects components at level $r$ if they are part of the same connected component at
some level $r'>r$, where $r'$ is a function of a tuning parameter $\tilde{\epsilon}\geq 0$.
The larger $\tilde{\epsilon}$ is, the more aggressive the pruning.

\begin{figure}
\framebox[6.5in]{
\begin{minipage}[t]{6in}
\noindent
{\bf Pruning of level $r$.}\\
\begin{itemize}
\item Set
$$\tilde{\lambda}_r = \frac{1}{v_d r^d}\paren{\frac{k}{n} - \frac{C_\delta}{n}\sqrt{kd\log n}} - \tilde{\epsilon}.$$
\item Connect any two components of $\C_n(r)$ that belong to the same connected component in $\C_n(r(\max(\tilde{\lambda}_r,0)))$.
\end{itemize}
\end{minipage}
}
\caption{An algorithm for pruning $G_r$ or $G^\nn_r$, applied for every $r$.
It assumes a tuning parameter $\tilde{\epsilon}>0$. Recall from Definition~\ref{defn:rlambda}
that for any $\lambda>0$, we take $r(\lambda)$ to be the value of $r$ for which
$v_dr^d {\lambda} = \frac{k}{n} + \frac{C_\delta}{n}\sqrt{kd\log n}$.}
\label{fig:alg_pruning}
\end{figure}

The pruning procedure builds upon the same intuition as for the procedure of \cite{KV11},
however it differs in its ability to handle either cluster-tree algorithms,
and works under significantly milder conditions than that of
\cite{KV11}. The intuition is the following. Suppose $A_n, A'_n\subset X_n$ are not
connected at some level $r$ in the empirical tree (before pruning), but ought to be: they
belong to the same connected component $A$ of $\C(\lambda)$, where
$\lambda = \min\braces{f(x): x\in A_n \cup A_n'}$. Then, key sample points from $A$ that would
have connected them are missing at level $r$ in the empirical tree
(Figure \ref{fig:falseClusters}).
These points have $r_k(x)$ greater than $r$, but probably not much greater.
Looking at a nearby level $r'>r$, we will find $A_n, A_n'$ connected and thus detect the situation.

The above intuition is likely to extend to cluster tree procedures other than the ones
discussed here. The main requirement on the cluster tree estimate is that points in $A$ (as
discussed above) be connected at some nearby level in the tree.

\subsection{Separation}

The pruning procedure increases connectivity, but we must make sure that it isn't too
zealous in doing so: clusters that are sufficiently separated should not be merged.
We now will require a bit more separation between two sets $A$ and $A'$
in order to keep them apart in the empirical tree. As might be expected, how much more
separation depends on the pruning parameter $\tilde{\epsilon}$.
The higher $\tilde{\epsilon}$, the more aggressive the pruning, and the greater the
separation requirement for detecting distinct clusters. The following lemma builds
on Corollary~\ref{cor:separation}.

\begin{lemma}
 Assume $E_o$. Consider two sets $A, A' \subset \X$, and let
$\lambda = \inf_{x \in A_\sigma \cup A'_\sigma} f(x)$.  Suppose there exists a separator set $S$ such that
\begin{itemize}
\item Any path in $\X$ from $A$ to $A'$ intersects $S$.
\item $\sup_{x \in S_\sigma} f(x) < (1-2\epsilon)\lambda - \tilde{\epsilon}$.
\end{itemize}
 Then $A\cap X_n$ and $A'\cap X_n$ are in separate connected components
of $\C_n(r(\lambda))$ after pruning, provided $k \geq 4C_\delta^2 (d/\epsilon^2) \log n$ and
$$ v_d\paren{2\sigma/(\alpha + 2)}^d\paren{(1-\epsilon)\lambda - \tilde{\epsilon}} > \frac{k}{n} + \frac{C_\delta}{n}\sqrt{kd\log n}.$$
\end{lemma}
\begin{proof}
Let $r$ denote $r(\lambda)$ and recall from the definitions of $r(\lambda)$ and $\tilde{\lambda}_r$
(Figure~\ref{fig:alg_pruning})
that
\begin{align*}
 \tilde{\lambda}_r
&= \lambda \paren{\frac{k}{n} - \frac{C_\delta}{n}\sqrt{kd\log n}}\paren{\frac{k}{n} + \frac{C_\delta}{n}\sqrt{kd\log n}}^{-1} -\tilde{\epsilon}\\
&\geq \paren{1-2\frac{C_\delta \sqrt{d\log n}}{\sqrt{k}}}\lambda -\tilde{\epsilon} \geq (1- \epsilon)\lambda - \tilde{\epsilon}.
\end{align*}
The final term, call it $\lambda'$, is $\geq 0$ by the hypotheses of the lemma.
Since $\tilde{\lambda}_r \geq \lambda'$ we have $r(\lambda')\geq r(\tilde{\lambda}_r)$.
Thus we just have to show that $A\cap X_n$ and $A'\cap X_n$ are in separate connected components of
$\C_n(r(\lambda'))$. To this end, notice that, under our assumptions on $A$ and $A'$,
these two sets belong to separate connected components of $\braces{x\in \X: f(x)\geq\lambda'}$;
in fact
\begin{eqnarray*}
\sup_{x \in S_\sigma} f(x)\leq (1-2\epsilon)\lambda -\tilde{\epsilon} \leq (1-\epsilon)\lambda'.
\end{eqnarray*}
Moreover, the final requirement of the lemma statement can be rewritten as
$r(\lambda') <  2\sigma/(\alpha+2)$.
The argument of Lemma~\ref{lemma:separation}(c) then implies that $A\cap X_n$ is disconnected
from $A'\cap X_n$ in $\C_n(r(\lambda'))$ and thus in $\C_n(r(\tilde{\lambda}_r))$,
and hence also at level $r$ after pruning.
\end{proof}

\subsection{Connectedness}

We now turn to the main result of this section, namely that the pruning procedure reconnects
incorrectly fragmented clusters. Recall the intuition detailed above. We first have to argue
that points with similar density make their first appearance at nearby levels $r$ of the
empirical tree. From the analysis of the previous sections, we know that a point $x$ is
present at level $r(f(x))$, roughly speaking. We now need to show that it cannot appear
at a level too much smaller than this.

These assertions about single points are true only if the density doesn't vary too dramatically
in their vicinity. In what follows, we will quantify the smoothness at scale $\sigma$ by the
constant
$$ L_\sigma \ = \ \sup_{\norm{x-x'}\leq \sigma} \abs{f(x) - f(x')} .$$

\begin{lemma}
\label{lem:pruning0}
Assume $E_o$.
Pick any $x$ and let $f_\sigma(x) = \inf_{x'\in B(x, \sigma)} f(x')$.
Suppose
$$v_d (\sigma/2)^d (f_\sigma(x) + L_\sigma) \geq \frac{k}{n} - \frac{C_\delta}{n}\sqrt{kd\log n}, $$
we then have
$$v_d r_k^d(x) (f_\sigma(x) + L_\sigma)\geq  \frac{k}{n} - \frac{C_\delta}{n}\sqrt{kd\log n}.$$
\end{lemma}
\begin{proof} Consider any $r$ such that
\begin{align*}
v_d r^d (f_\sigma(x) + L_\sigma)<  \frac{k}{n} - \frac{C_\delta}{n}\sqrt{kd\log n}\leq v_d (\sigma/2)^d
(f_\sigma(x) + L_\sigma).
\end{align*}
Then $r\leq \sigma/2$, implying $f(B(x, r)) \leq v_d r^d (f_\sigma(x) + L_\sigma)$.
Using the first inequality and Lemma \ref{lemma:convergence-balls}, we have $f_n(B(x, r))<k/n$, that is $r<r_k(x)$.
\end{proof}

Next, by combining the above lower-bound on $r_k(x)$ with our previous results on connectedness for
both types of algorithms, we obtain the following pruning guarantees.

\begin{lemma}
\label{lem:pruning}
Assume that event $E_o$ holds, and that $\tilde{\epsilon}\geq L_\sigma$.
Let $A_n$ and $A_n'$ denote two disconnected sets of vertices of $G_r$ or $G^\nn_r$
after pruning, for some $r > 0$. Define $\lambda = \inf_{x\in A_n\cup A_n'} f(x)$.
Then $A_n$ and $A_n'$ are disconnected in the level set
$\braces{x\in \X: f(x) \geq \lambda}$ if the following two conditions hold: first,
$$ v_d (\sigma/2)^d (\lambda - L_\sigma) \ \geq \ \frac{k}{n} + \frac{C_\delta}{n}\sqrt{k d \log n},$$
and second,
$$ k \ \geq \
\left\{
\begin{array}{ll}
4C_\delta d\log n & \mbox{for $G_r$} \\
\max(4 C_\delta^2 d \log n, (\Lambda/\lambda)8C_\delta d\log n) & \mbox{for $G^\nn_r$}
\end{array}
\right.
$$
\end{lemma}
\begin{proof}
Let $A$ be any connected component of $\{x\in \X: f(x)\geq\lambda\}$. We'll show that $A \cap X_n$ is
connected in $G_r$ (or $G^\nn_r$) after pruning, from which the lemma follows immediately.

Define $\lambda_\sigma = \inf_{x\in A_\sigma} f(x) \geq \inf_{x \in A} f(x) - L_\sigma \geq \lambda - L_\sigma$.
Recall from Definition~\ref{defn:rlambda} that $r(\lambda_\sigma)$ is the value of $r$ for which
$v_d r^d \lambda_\sigma = \frac{k}{n} + \frac{C_\delta}{n}\sqrt{k d \log n}$.
The first condition in the lemma statement thus implies that $r(\lambda_\sigma)\leq \sigma/2$.
The second condition, together with Theorem~\ref{thm:connectedness} or
Theorem~\ref{thm:connectedness2}, implies that $A \cap X_n$ is connected at level
$r(\lambda_\sigma)$ of $G$ or $G^\nn$.

Next we show that $r(\lambda_\sigma)\leq r(\tilde{\lambda}_r)$, by showing that
$\lambda_\sigma\geq\tilde{\lambda}_r$.
Again by the first condition on $k$, Lemma \ref{lem:pruning0}
holds for every $x\in A$, implying with little effort that
\begin{align*}
 \lambda_\sigma\geq  \frac{1}{v_d r^d}\paren{\frac{k}{n} - \frac{C_\delta}{n}\sqrt{kd\log n}} -L_\sigma
\geq \frac{1}{v_d r^d}\paren{\frac{k}{n} - \frac{C_\delta}{n}\sqrt{kd\log n}} -\tilde{\epsilon}= \tilde{\lambda}_r.
\end{align*}

Thus $A \cap X_n$ is connected at level $r(\tilde{\lambda}_r) \geq r(\lambda_\sigma)$
of $G$ (or $G^\nn$), and thus is reconnected when pruning at level $r$.
\end{proof}

The separation and connectedness results of this section can now be combined into the
following theorem.

\begin{thm}
 There is an absolute constant $C$ such that the following holds.
Pick any $0<\delta, \epsilon< 1$ and $\tilde{\epsilon} > 0$. Assume
Algorithm 1 or 2 is run on a sample $X_n$ of size $n$ drawn from $f$,
with settings
$$\sqrt{2}\leq \alpha\leq 2 \text{\ \ and\ \ }
k\geq C\cdot\frac{d \log n}{\epsilon^2}\cdot \log^2 \frac{1}{\delta},$$
followed by the pruning procedure with parameter $\tilde{\epsilon}$.

Then the following holds with probability at least $1-\delta$. Define
$$
\lambda_o
\ = \
\frac{k}{n v_d (\sigma/2)^d} \cdot \frac{1+\epsilon}{1-\epsilon} + \frac{\tilde{\epsilon}}{1-\epsilon},
$$
or in the case of Algorithm 2, the maximum of this quantity and
$(\Lambda/k) C d \log n\cdot \log (1/\delta)$, where $\Lambda = \sup_{x \in \X} f(x)$.

\emph{Recovery of true clusters:}
Consider any two sets $A, A' \subset \X$, and suppose
$\lambda = \inf_{x \in A_\sigma \cup A'_\sigma} f(x) \geq \lambda_o$.
Suppose there exists a set $S$ such that
\begin{itemize}
\item Any path in $\X$ from $A$ to $A'$ intersects $S$.
\item $\sup_{x \in S_\sigma} f(x) < (1-2\epsilon)\lambda - \tilde{\epsilon}$.
\end{itemize}
Then $A\cap X_n$ and $A'\cap X_n$ are individually connected in $\C_n(r(\lambda))$, but lie
in two separate connected components.

\emph{Removal of false clusters:} Assume the pruning parameter satisfies
$\tilde{\epsilon} \geq 2 \sup_{\norm{x-x'}\leq \sigma} \abs{f(x) - f(x')}$.
Let $A_n$ and $A_n'$ denote the vertices of two disjoint connected components in $\C_n(r)$,
for any $r>0$. If $\lambda = \inf_{x\in A_n\cup A_n'} f(x) \geq \lambda_o$, then
the two sets of points $A_n$ and $A_n'$ are disconnected in the level
set $\braces{x\in \X: f(x) \geq \lambda}$.
\label{thm:pruning}
\end{thm}

The first part of the above theorem (recovery of true clusters) implies that the pruned tree
remains a consistent estimator of the cluster tree, under the same asymptotic conditions as
those for Theorem \ref{thm:upper-bound} and Theorem \ref{thm:upper-bound2},
and the additional condition that $\tilde{\epsilon}\to 0$.

The second part of the theorem states some general conditions on $\tilde{\epsilon}$ and $\lambda$
under which false clusters are removed.
To better understand these conditions, let's consider the simple case when $f$ is
H\"{o}lder-smooth:
\begin{align*}
\exists L, \beta>0 \text{ such that } \forall x, x' \in \X, \quad \abs{f(x) - f(x')} \leq L\norm{x-x'}^\beta.
\end{align*}
Consider $\sigma = (\tilde{\epsilon}/L)^{1/\beta}$ so that we have
$\sup_{\norm{x-x'}\leq \sigma} \abs{f(x) - f(x')}\leq  \tilde{\epsilon}$.
Consider $0<\epsilon<1/3$. Then any $\lambda>4\tilde{\epsilon}$ is $\geq \lambda_o$
if $k$ is in the range
\begin{align*}
\frac{\Lambda}{4\tilde{\epsilon}}\cdot C\cdot\frac{d \log n}{\epsilon^2}\cdot \log^2 \frac{1}{\delta}\leq k
\leq 2^{-d}\cdot v_d \cdot L ^{-d/\beta}\cdot \tilde{\epsilon}^{(\beta + d)/\beta}\cdot n.
\end{align*}
Note that, without knowing the H\"{o}lder parameters $L$ and $\beta$, we can ensure
$k$ is in the above range for any particular $0<\epsilon<1/3$, provided $n$ and $k = k(n)$ are sufficiently large,
by choosing $\tilde{\epsilon}$ as a function of $k$ (e.g. $k=\Theta(\log^3 n)$ and $\tilde{\epsilon}= \Theta(1/\sqrt{k})$).

Finally, remark that under the above smoothness assumption and choice of $\sigma, \epsilon$, we can further guarantee
that \emph{all} false clusters are removed!
We only need to reconnect all components at levels where the minimum $f$ value is at most
$4\tilde{\epsilon}$.
By Lemma \ref{lem:pruning0}, for $k$ in the above range,  we have
$r_k(x) \geq (k/10nv_d\tilde{\epsilon})^{1/d}$ when $f(x)\leq 4\tilde{\epsilon}$.
Thus, we just need to reconnect all components at levels $r> (k/10nv_d\tilde{\epsilon})^{1/d}$,
and prune all other levels as discussed above.
This then guarantees that all false clusters are removed with high probability,
while also ensuring that the estimator remains consistent.

\section{Final remarks}

Both cluster tree algorithms are variations on standard estimators, but
carefully control the neighborhood size $k$ and make use of a novel parameter
$\alpha$ to allow more edges at every scale $r$. The analysis relies on
$\alpha$ being at least $\sqrt{2}$, and on
$k$ being at least $d \log n$. Is it possible to dispense with $\alpha$
(that is, to use $\alpha = 1$) while maintaining this setting of $k$?

There remains a discrepancy of $2^d$ between the upper and lower bounds on
the sample complexity of building a hierarchical clustering that distinguishes
all $(\sigma, \epsilon)$-separated clusters. Can this gap be closed, and if
so, what is needed, a better analysis or a better algorithm?

Finally, unlike with plug-in estimators of the cluster tree, our algorithms encode no
knowledge of the dimension of the support. It is therefore likely that our results extend
to settings where the distribution is supported on a low-dimensional subspace of $\R^d$.

\appendix

\section{Plug-in estimation of the cluster tree}

One way to build a cluster tree is to return $\C_{f_n}$, where $f_n$
is a uniformly consistent density estimate.
\begin{lemma}
Suppose estimator $f_n$ of density $f$ (on space $\X$) satisfies
$ \sup_{x \in \X} |f_n(x) - f(x)| \leq \epsilon_n .$
Pick any two disjoint sets $A,A' \subset \X$ and define
$\Xi =  \inf_{x \in A \cup A'} f(x)$ and $\xi  = \sup_{A \stackrel{P}{\leadsto} A'} \inf_{x \in P} f(x)$.
If $\Xi - \xi > 2 \epsilon_n$ then $A,A'$ lie entirely in disjoint connected
components of $\C_{f_n}(\Xi - \epsilon_n)$.
\end{lemma}
\begin{proof}
$A$ and $A'$ are each connected in $\C_{f_n}(\Xi - \epsilon_n)$. But there is no
path from $A$ to $A'$ in $\C_{f_n}(\lambda)$ for $\lambda > \xi + \epsilon_n$.
\end{proof}

The problem, however, is that computing the level sets of $f_n$ is usually
not an easy task. Hence we adopt a different approach in this paper.

\section{Consistency}

The following is a straightforward exercise in analysis.
\begin{lemma}
Suppose density $f: \R^d \rightarrow \R$ is continuous and is zero outside a
compact subset $\X \subset \R^d$. Suppose further that for some $\lambda$,
$\{x \in \X: f(x) \geq \lambda\}$ has finitely many connected components, among
them $A \neq A'$. Then there exist $\sigma, \epsilon > 0$ such that $A$ and $A'$
are $(\sigma,\epsilon)$-separated.
\end{lemma}
\begin{proof}
Let $A_1, A_2, \ldots, A_k$ be the connected components of $\{f \geq \lambda\}$,
with $A = A_1$ and $A' = A_2$.

First, each $A_i$ is closed and thus compact. To see this, pick any
$x \in \X \setminus A_i$. There must be some $x'$ on the shortest path from
$x$ to $A_i$ with $f(x') < \lambda$ (otherwise $x \in A_i$). By continuity
of $f$, there is some ball $B(x',r)$ on which $f < \lambda$; thus this ball
doesn't touch $A_i$. Then $B(x,r)$ doesn't touch $A_i$.

Next, for any $i \neq j$, define $\Delta_{ij} = \inf_{x \in A_i, y \in A_j} \|x - y\|$
to be the distance between $A_i$ and $A_j$. We'll see that $\Delta_{ij} > 0$.
Specifically, define $g: A_i \times A_j \rightarrow \R$ by $g(a,a') = \|a-a'\|$.
Since $g$ has compact domain, it attains its infimum for some
$a \in A_i, a' \in A_j$. Thus $\Delta_{ij} = \|a-a'\| > 0$.

Let $\Delta = \min_{i \neq j} \Delta_{ij} > 0$, and define $S$ to be the set
of points at distance exactly $\Delta/2$ from $A$:
$S = \{x \in \X: \inf_{y \in A} \|x-y\| = \Delta/2 \} .$
$S$ separates $A$ from $A'$. Moreover, it is closed by continuity of
$\|\cdot \|$, and hence is compact. Define $\lambda_o = \sup_{x \in S} f(x)$.
Since $S$ is compact, $f$ (restricted to $S$) is maximized at some $x_o \in S$.
Then $\lambda_o = f(x_o) < \lambda$.

To finish up, set $\delta = (\lambda - \lambda_o)/3 > 0$. By uniform continuity
of $f$, there is some $\sigma > 0$ such that $f$ doesn't change by more than
$\delta$ on balls of radius $\sigma$. Then
$f(x) \leq \lambda_o + \delta = \lambda - 2\delta$ for $x \in S_\sigma$ and
$f(x) \geq \lambda - \delta$ for $x \in A_\sigma \cup A'_\sigma$.

Thus $S$ is a $(\sigma,\delta/(\lambda-\delta))$-separator for $A,A'$.
\end{proof}

\section{Proof details}

\subsection{Proof of Lemma~\ref{lemma:convergence-balls}}

We start with a standard generalization result due to Vapnik and Chervonenkis;
the following version is a paraphrase of Theorem 5.1 of \citet{BBL04}.
\begin{thm}
Let $\G$ be a class of functions from $\X$ to $\{0,1\}$ with VC dimension $d < \infty$,
and $\P$ a probability distribution on $\X$. Let $\E$ denote expectation with respect
to $\P$. Suppose $n$ points are drawn independently at random from $\P$; let $\E_n$
denote expectation with respect to this sample. Then for any $\delta > 0$, with
probability at least $1-\delta$, the following holds for all $g \in \G$:
$$
- \min(\beta_n \sqrt{\E_n g}, \beta_n^2 + \beta_n \sqrt{\E g})
\ \ \leq \ \
\E g - \E_n g
\ \ \leq \ \
\min(\beta_n^2 + \beta_n \sqrt{\E_n g}, \beta_n \sqrt{\E g}),
$$
where $\beta_n = \sqrt{(4/n)(d \ln 2n + \ln (8/\delta))}$.
\label{thm:vc}
\end{thm}
By applying this bound to the class $\G$ of indicator functions over balls (or half-balls),
we get the following:
\begin{lemma}
Suppose $X_n$ is a sample of $n$ points drawn independently at random from a
distribution $f$ over $\X$. For any set $Y \subset \X$, define
$f_n(Y) = |X_n \cap Y|/n$. There is a universal constant $C_o > 0$ such that
for any $\delta > 0$, with probability at least $1-\delta$, for any
ball (or half-ball) $B \subset \R^d$,
\begin{eqnarray*}
f(B) \geq \frac{C_o}{n} \left(d \log n +\log \frac{1}{\delta} \right)
& \implies & f_n(B) > 0 \\
f(B) \geq \frac{k}{n} + \frac{C_o}{n} \left(d \log n + \log \frac{1}{\delta} + \sqrt{k \left( d \log n + \log \frac{1}{\delta}\right)}  \right) & \implies & f_n(B) \geq \frac{k}{n} \\
f(B) < \frac{k}{n} - \frac{C_o}{n} \left(d \log n + \log \frac{1}{\delta} + \sqrt{k \left( d \log n + \log \frac{1}{\delta}\right)}  \right) & \implies & f_n(B) < \frac{k}{n}
\end{eqnarray*}
\label{lemma:convergence-balls1}
\end{lemma}
\begin{proof}
The VC dimension of balls in $\R^d$ is $d+1$, while that of half-balls (each the
intersection of a ball and a halfspace) is $O(d)$. The following statements apply
to either class.

The bound $f(B) - f_n(B) \leq \beta_n \sqrt{f(B)}$ from Theorem~\ref{thm:vc} yields
$f(B) > \beta_n^2  \implies f_n(B) > 0.$
For the second bound, we use $f(B) - f_n(B) \leq \beta_n^2 + \beta_n \sqrt{f_n(B)}$.
It follows that
\begin{eqnarray*}
f(B) \geq \frac{k}{n} + \beta_n^2 + \beta_n \sqrt{\frac{k}{n}} & \implies & f_n(B) \geq \frac{k}{n} .
\end{eqnarray*}
For the last bound, we rearrange $f(B) - f_n(B) \geq - (\beta_n^2 + \beta_n \sqrt{f(B)})$
to get
\begin{eqnarray*}
f(B) < \frac{k}{n} - \beta_n^2 - \beta_n \sqrt{\frac{k}{n}} & \implies & f_n(B) < \frac{k}{n} .
\end{eqnarray*}
\end{proof}

Lemma~\ref{lemma:convergence-balls} now follows immediately, by taking
$k \geq d \log n$. Since the uniform convergence bounds have error bars of
magnitude $(d \log n)/n$, it doesn't make sense, when using them,
to take $k$ any smaller than this.

\subsection{Proof of Lemma~\ref{lemma:connectedness}}

Consider any $x, x' \in A \cap X_n$. Since $A$ is connected, there is a path $P$ in
$A$ with $x \stackrel{P}{\leadsto} x'$. Fix any $0 < \gamma < 1$. Because the density
of $A_\sigma$ is lower bounded away from zero, it follows by a volume and packing-covering
argument that $A$, and thus $P$, can be covered by a finite number of balls of
diameter $\gamma r$. Thus we can choose finitely many points
$z_1, z_2, \ldots, z_k \in P$ such that $x = z_0$, $x' = z_k$ and
$ \| z_{i+1} - z_i \| \leq \gamma r .$

Under $E_o$ (Lemma~\ref{lemma:convergence-balls}), any ball centered in $A$ with radius
$(\alpha-\gamma)r/2$ contains at least one data point if
\begin{equation}
v_d \left(\frac{(\alpha-\gamma)r}{2} \right)^d \lambda
\ \geq \  \frac{C_\delta d \log n}{n} . \label{eq:r}
\end{equation}
Assume for the moment that this holds. Then, every ball $B(z_i, (\alpha-\gamma)r/2)$
contains at least one point; call it $x_i$.

By the upper bound on $r$, each such $x_i$ lies in $A_{\sigma-r}$; therefore,
by Lemma~\ref{lemma:convergence-balls}, the $x_i$ are all active in $G_r$. Moreover,
consecutive points $x_i$ are close together:
$$ \|x_{i+1} - x_i \|
\ \leq \
\|x_{i+1} - z_{i+1}\| + \|z_{i+1} -z_i\| + \|z_i - x_i \|
\ \leq \
\alpha r .$$
Thus all edges $(x_i, x_{i+1})$ exist in $G_r$, whereby $x$ is connected to
$x'$ in $G_r$.

All this assumes that equation (\ref{eq:r}) holds for some $\gamma > 0$.
Taking $\gamma \rightarrow 0$ gives the lemma.

\section{Fano's inequality}
\label{sec:fano}

Consider the following game played with a predefined, finite class of distributions
$F = \{f_1, \ldots, f_\ell\}$, defined on a common space $\X$:
\begin{itemize}
\item Nature picks $I \in \{1,2,\ldots, \ell\}$.
\item Player is given $n$ i.i.d.\ samples $X_1, \ldots, X_n$ from $f_i$.
\item Player then guesses the identity of $I$.
\end{itemize}
Fano's inequality \citep{CT90, Y97} gives a lower bound on the number of samples $n$
needed to achieve a certain success probability. It depends on how similar the
distributions $f_i$ are: the more similar, the more samples are needed.
Define
$ \theta = \frac{1}{\ell^2} \sum_{i,j = 1}^\ell K(f_i, f_j) $
where $K(\cdot)$ is KL divergence. Then $n$ needs to be $\Omega((\log \ell)/\theta)$.
Here's the formal statement.
\begin{thm}[Fano]
Let $g: \X^n \rightarrow \{1,2,\ldots,\ell\}$ denote Player's computation.
If Nature chooses $I$ uniformly at random from $\{1,2,\ldots, \ell\}$, then for any
$0 < \delta < 1$,
\begin{eqnarray*}
n \leq \frac{(1-\delta) (\log_2 \ell) - 1}{\theta} & \implies & \mbox{\rm Pr}(g(X_1, \ldots, X_n) \neq I) \geq \delta.
\end{eqnarray*}
\end{thm}

\subsection*{Acknowledgements}

Dasgupta is grateful to the National Science Foundation for support under grant IIS-0347646, and von Luxburg acknowledges funding of the German Research Foundation
(individual grant LU1718/1-1 and Research Unit 1735 "Structural Inference in Statistics: Adaptation and Efficiency'').

\bibliography{sanjoy}

\end{document}